\documentclass{article}
\usepackage{latexsym,amsmath,amssymb,stmaryrd,graphicx,hyperref}
\usepackage{amsthm,amsmath,amssymb}
\usepackage{mathabx}
\usepackage{geometry}

\theoremstyle{plain}
    \newtheorem{thm}{Theorem}
    
    \newtheorem{lem}[thm]{Lemma}
    \newtheorem{prop}[thm]{Proposition}
    
    \newtheorem{fact}[thm]{Fact}

    \newtheorem{ex}[thm]{Example}
    
\theoremstyle{definition}
    \newtheorem{defn}[thm]{Definition}
    
\theoremstyle{remark}

  \newcommand{\ignore}[1]{}
\DeclareMathOperator{\Csp}{CSP}
\DeclareMathOperator{\Aut}{Aut}

\newcommand{\disj}{\mathrel{||}}

\title{Tractable Set Constraints}
\author{Manuel Bodirsky\thanks{Manuel Bodirsky has received funding from the ERC 
under the European Community's Seventh Framework Programme
(FP7/2007-2013 Grant Agreement no. 257039).
} \\ \'{E}cole Polytechnique, LIX \\ (UMR 7161 du CNRS), \\ 91128 Palaiseau, France \\
\and 
Martin Hils \\ 
Institut de Math\'ematiques de Jussieu \\ (UMR 7586 du CNRS), \\
Universit\'e Paris Diderot Paris 7, 
UFR de Math\'ematiques, \\
75205 Paris Cedex 13, France
}

\begin{document}
\maketitle
\begin{abstract}
Many fundamental problems in artificial intelligence, knowledge representation, and verification involve reasoning about sets
and relations between sets and can be modeled as \emph{set constraint satisfaction problems (set CSPs)}.
Such problems are frequently intractable, but there are several important set CSPs
that are known to be polynomial-time tractable.
We introduce a large class of set CSPs that can be solved in quadratic time.
Our class, which we call $\mathcal{EI}$, contains all previously known tractable set CSPs,
but also some new ones that are of crucial importance for example in description logics.
The class of $\mathcal{EI}$ set constraints has an elegant universal-algebraic characterization,
which we use to show that every set constraint language that properly contains
all $\mathcal{EI}$ set constraints already has a finite sublanguage with an NP-hard constraint satisfaction problem.
\end{abstract}

\noindent 
An extended abstract of the content of this article
appeared in the proceedings of IJCAI'11~\cite{BodHilsKrim}. 

\section{Introduction}
\emph{Constraint satisfaction problems} are computational problems
where, informally, the input consists of a finite set of variables
and a finite set of constraints imposed on those variables; the task
is to decide whether there is an assignment of 
values to the variables such that all the constraints
are simultaneously satisfied. \emph{Set constraint satisfaction problems} are special constraint satisfaction problems where the values are sets,
and the constraints might, for instance,  force that one set $y$ includes another set $x$,
or that one set $x$ is disjoint to another set $y$. 
The constraints might also be ternary, such as the constraint 
that the intersection of two sets $x$ and $y$ is contained in $z$, in symbols $(x \cap y) \subseteq z$. 

To systematically study the computational complexity of constraint
satisfaction problems, it has turned out to be a fruitful approach
to consider constraint satisfaction problems $\Csp(\Gamma)$ where the set of allowed
constraints is formed from a fixed set $\Gamma$ of relations $R \subseteq D^k$ over a common domain $D$. 
This way of parametrizing the constraint satisfaction problem by a  \emph{constraint language} $\Gamma$ has led to many strong algorithmic results~\cite{Maltsev,IMMVW,BoundedWidth,BodirskyKutzAI,tcsps-journal}, 
and to many powerful hardness conditions 
for large classes of constraint satisfaction problems~\cite{Schaefer,JBK,conservative,Bulatov,tcsps-journal}.

A \emph{set constraint language} $\Gamma$ is a set of relations $R \subseteq ({\cal P}({\mathbb N}))^k$ where the common domain $D = {\cal P}({\mathbb N})$ is the set of all 
subsets of the natural numbers; moreover, we require that each relation $R$
can be defined by a Boolean combination of equations
over the signature $\sqcap$, $\sqcup$, $c$, $\bf 0$, and $\bf 1$, which are 
function symbols for intersection, union, complementation, the empty and full set, respectively. 
Details of the formal definition and many examples of set constraint languages
can be found in Section~\ref{sect:setcsps}.
The choice of ${\mathbb N}$ is just for notational convenience; as we will see, we could have selected any infinite set for our purposes. In the following, a \emph{set constraint satisfaction problem (set CSP)} is a problem of the form $\Csp(\Gamma)$ for
a set constraint language $\Gamma$. 
It has been shown by Marriott and Odersky~\cite{MarriottOdersky} that all set CSPs are contained in NP; they also showed that the largest
set constraint language, which consists of \emph{all} relations
that can be defined as described above, has an NP-hard set CSP.

Drakengren and Jonsson~\cite{DrakengrenJonssonSets} initiated the
search for set CSPs that can be solved
in polynomial time. They showed that $\Csp(\{\subseteq,\disj,\neq\})$ can be solved in polynomial time, where 
\begin{itemize}
\item $x \subseteq y$ holds iff $x$ is a subset of or equal to $y$;
\item $x \disj y$ holds iff $x$ and $y$ are disjoint sets; and
\item $x \neq y$ holds iff $x$ and $y$ are distinct sets.
\end{itemize}
They also showed that $\Csp(\Gamma)$ can be solved in polynomial time if 
all relations in $\Gamma$ can be defined by formulas of the form
$$ x_1 \neq y_1 \vee \dots \vee x_k \neq y_k \vee x_0 \subseteq y_0$$
or of the form
$$ x_1 \neq y_1 \vee \dots \vee x_k \neq y_k \vee x_0 \disj y_0$$
where $x_0,\dots,x_k,y_0,\dots,y_k$ are not necessarily distinct variables.
We will call the set of all relations that can be defined in this way \emph{Drakengren and Jonsson's set constraint language}.
It is easy to see that the algorithm they present runs in time quadratic 
in the size of the input. 

On the other hand, Drakengren and Jonsson~\cite{DrakengrenJonssonSets} show that if $\Gamma$ contains the relations defined by formulas of the form
$$ x_1 \neq y_1 \vee \dots \vee x_k \neq y_k \vee u_1 \disj v_1 \vee \dots \vee u_k \disj v_k$$
the problem $\Csp(\Gamma)$ is NP-hard. 

\paragraph{Contributions and Outline.}
We present a significant extension of Drakengren and Jonsson's 
set constraint language (Section \ref{sect:setcsps})
whose CSP can still be solved in quadratic time in the input size (Section~\ref{sect:algorithm}); we call this set constraint language 
$\mathcal{EI}$. 
Unlike Drakengren and Jonsson's set constraint language, our language also
contains the ternary relation defined by $(x \cap y) \subseteq z$, which is a
relation that is of particular interest in description logics -- we will discuss this below. Moreover, we show that any further extension of 
$\mathcal{EI}$
contains a finite sublanguage with an
 NP-hard set CSP (Section~\ref{sect:maximal}), using concepts from model theory and universal algebra.
 In this sense, we present a \emph{maximal} tractable class of set constraint satisfaction problems.
 
Our algorithm is based on the concept of \emph{independence} in constraint languages which was discovered several times independently in the 
90's~\cite{Koubarakis,JonssonBaeckstroem,MarriottOdersky} 
-- see also~\cite{BroxvallJonssonRenz,Disj}; however,
we apply this concept \emph{twice} in a novel, nested way, which leads to a
two level resolution procedure that can be implemented to run in quadratic time.
The technique we use to prove the correctness of the algorithm is also an important contribution of our paper, and we believe that a similar approach can be applied in 
many other contexts; our technique is inspired by the already
mentioned connection to universal algebra.

\subsection*{Application Areas and Related Literature}
\paragraph{Set Constraints for Programming Languages.}
Set constraints find applications in program analysis;
here, a set constraint is of the form $X \subseteq Y$, where $X$
and $Y$ are set expressions. Examples of set expressions are $\bf 0$ (denoting the empty set), set-valued variables, and union and intersection of sets, but also expressions of the form $f(Z_1,Z_2)$
where $f$ is a function symbol and $Z_1, Z_2$ are again set expressions. Unfortunately, the worst-case complexity of most
of the reasoning tasks considered in this setting is \emph{very} high,
often EXPTIME-hard; see~\cite{Aiken94} for a survey.
More recently, it has been shown that the
quantifier-free combination of set constraints (without function symbols) and cardinality constraints (quantifier-free Pressburger arithmetic) has a satisfiability problem in NP~\cite{KuncakRinard}. 
This logic
(called QFBAPA) is interesting for program verification~\cite{KuncakNguyenRinard}. 

\paragraph{Tractable Description Logics.}
Description logics are a family of knowledge representation 
formalisms that can be used to formalize and reason with concept
definitions. The computational complexity of most of the computational tasks that
have been studied for the various formalisms is usually quite high.
However, in the last years a series of description logics 
(for example $\mathcal{EL}$, $\mathcal{EL}^{++}$, Horn-${\mathcal FL}_0$, 
and various extensions and fragments~\cite{KuestersMolitor,BaaderEL,BaaderBrandtLutz,KroetzschRudolphHitzler}) has been
discovered where crucial tasks such as e.g.\  entailment, concept satisfiability and knowledge base satisfiability can be decided in polynomial time. 

Two of the basic assertions that can be made in $\mathcal{EL}^{++}$
and Horn-${\mathcal FL}_0$
are $C_1 || C_2$ (\emph{there is no $C_1$ that is also $C_2$}) and
$C_1 \cap C_2 \subseteq C_3$ (\emph{every $C_1$ that is $C_2$ is also $C_3$}), for concept names $C_1, C_2, C_3$. These are 
$\mathcal{EI}$ set constraints, and the latter has not been treated in the framework of Drakengren and Jonsson. None of the
description logics with a tractable knowledge base satisfiability problem contains all $\mathcal{EI}$ set constraints.




\paragraph{Spatial Reasoning.} 
Several spatial reasoning formalisms (like RCC-5 and RCC-8) are closely related to set constraint satisfaction problems. These formalisms allow to reason about relations between \emph{regions};
in the fundamental formalism RCC-5 (see e.g. \cite{RCC5JD}), one can think of a region
as a non-empty set, and possible (binary) relationships are containment, 
disjointness, equality, overlap, and disjunctive combinations thereof.
Thus, the exclusion of the empty set
is the most prominent difference between the set constraint languages 
studied by Drakengren and Jonsson in~\cite{DrakengrenJonssonSets} (which are contained in the class of set constraint languages considered here), and 
RCC-5 and its fragments.


\section{Constraint Satisfaction Problems}
To use existing terminology in logic and model theory, it will be convenient to formalize constraint languages as (relational) \emph{structures} (see e.g.~\cite{HodgesLong}). A \emph{structure} $\Gamma$
 is a tuple $(D; f^{\Gamma}_1,f^{\Gamma}_2,\dots,R^{\Gamma}_1,R^{\Gamma}_2,\dots)$ where $D$ is a set (the \emph{domain} of $\Gamma$), each $f^{\Gamma}_i$ is a function from $D^{k_i} \rightarrow D$ (where $k_i$ is called the \emph{arity} of $f^{\Gamma}_i$), and each $R^{\Gamma}_i$ is 
 a relation over $D$, i.e., a subset of $D^{l_i}$ (where $l_i$ is called the \emph{arity} of $R^{\Gamma}_i$). For each function $f_i^{\Gamma}$ we assume that there is a \emph{function symbol} which we denote by $f_i$, and for each relation $R^{\Gamma}_i$
 we have a \emph{relation symbol} which we denote by $R_i$.
 Constant symbols will be treated as $0$-ary function symbols.
 The set $\tau$ of all relation and function symbols for some structure
 $\Gamma$ is called the \emph{signature} of $\Gamma$, and we also
say that $\Gamma$ is a \emph{$\tau$-structure}.
 If the signature of $\Gamma$ only contains relation symbols and no function symbols, we also say that $\Gamma$ is a \emph{relational structure}.
In the context of constraint satisfaction, relational
structures $\Gamma$ are also called \emph{constraint languages},
and a constraint language $\Gamma'$ is called a \emph{sublanguage}
(or \emph{reduct}) of a constraint language $\Gamma$ if the relations in $\Gamma'$
are a subset of the relations in $\Gamma$ (and $\Gamma$ is called an \emph{expansion} of $\Gamma'$).

Let $\Gamma$ be a relational structure with domain $D$ and a \emph{finite} signature $\tau$.
The \emph{constraint satisfaction problem} for $\Gamma$
is the following computational problem, also denoted by $\Csp(\Gamma)$: 
given a finite set of variables $V$ and a conjunction $\Phi$ of atomic formulas of the form $R(x_1,\dots,x_k)$, where $x_1,\dots,x_k \in V$ and $R \in \tau$, does there exists an assignment 
$s \colon V \rightarrow D$ such that
for every constraint $R(x_1,\dots,x_k)$ in the input
we have that $(s(x_1),\dots,s(x_k)) \in R^\Gamma$?

The mapping $s$ is also called a \emph{solution} to the \emph{instance} $\Phi$ of $\Csp(\Gamma)$, and the conjuncts of $\Phi$ are called \emph{constraints}.
Note that we only introduce constraint satisfaction problems
$\Csp(\Gamma)$ for \emph{finite constraint languages}, i.e., 
relational structures $\Gamma$ with a
\emph{finite} relational signature. 

\section{Set Constraint Languages}
\label{sect:setcsps}
In this section, we give formal definitions of 
set constraint languages.
Let ${\mathfrak S}$ 
be the structure with domain ${\cal P}({\mathbb N})$, the set of all subsets of natural numbers, 
and with signature $\{\sqcap,\sqcup,c,{\bf 0},{\bf 1}\}$, 
where 
\begin{itemize}
\item $\sqcap$ is a binary function symbol that denotes intersection, i.e., 
$\sqcap^{\mathfrak S}=\cap$;
\item $\sqcup$ is a binary function symbol for union, i.e., 
$\sqcup^{\mathfrak S}=\cup$;
\item $c$ is a unary function symbol for complementation, i.e., $c^{\mathfrak S}$ is the function
that maps $S \subseteq {\mathbb N}$ to ${\mathbb N} \setminus S$;
\item ${\bf 0}$ and ${\bf 1}$ are constants (treated as $0$-ary function symbols) denoting the empty set $\emptyset$ and the full set $\mathbb N$, respectively. 
\end{itemize}
Sometimes, we simply write $\sqcap$ for the function $\sqcap^{\mathfrak S}$ 
and $\sqcup$ for the function $\sqcup^{\mathfrak S}$,
i.e., we do not distinguish between a function symbol and the respective function.
We use the symbols $\sqcap,\sqcup$ and not the
symbols $\cap, \cup$ to prevent confusion with meta-mathematical usages of $\cap$ and $\cup$ in the text. 

A \emph{set constraint language} is a relational structure
with a set of relations with a quantifier-free 
first-order definition in $\mathfrak S$. We always allow equality in first-order formulas, and the equality symbol $=$ is always interpreted to be the true 
equality relation on the domain of the structure. 

\begin{ex}
The ternary relation $\big\{(x,y,z)  \in {\cal P}({\mathbb N})^3 \; | \; x \sqcap y \sqsubseteq z \big\}$ 
has the quantifier-free first-order definition $z \sqcap (x \sqcap y) = x \sqcap y$ over $\mathfrak S$.
\end{ex}

\begin{thm}[Follows from Proposition 5.8 in~\cite{MarriottOdersky}]
\label{thm:np}
Let $\Gamma$ be a set constraint language with a finite signature. Then $\Csp(\Gamma)$ is in NP.
\end{thm}

It is well-known that the structure $({\cal P}({\mathbb N}); \sqcup,\sqcap,c,{{\bf 0}},{{\bf 1}})$ is a Boolean algebra, with 
\begin{itemize}
\item ${\bf 0}$ playing the role of false, and ${\bf 1}$ playing the role of true; 
\item $c$ playing the role of $\neg$;
\item $\sqcap$ and $\sqcup$ playing the role of $\wedge$ and $\vee$, respectively.
\end{itemize}
To not confuse logical connectives with the connectives
of Boolean algebras, we always use the symbols $\sqcap$, $\sqcup$, and $c$ instead of the usual function symbols $\wedge$, $\vee$, and $\neg$ in Boolean algebras. To facilitate the notation, we also write
$\bar x$ instead of $c(x)$, and $x \neq y$ instead of $\neg (x = y)$.

We assume that all terms $t$ over the functional signature $\{\sqcap,\sqcup,c,{\bf 0},{\bf 1}\}$ are written in \emph{(inner) conjunctive normal form (CNF)}, i.e., as $t = \bigsqcap_{i=1}^n \bigsqcup_{j=1}^{n_i} l_{ij}$ where $l_{ij}$ is either
of the form $\bar x$ or of the form $x$ for a variable $x$ (every term over $\{\sqcap,\sqcup,c,{\bf 0},{\bf 1}\}$ can be re-written into an equivalent term of this form, using the usual laws of Boolean algebras~\cite{Boole}). We allow the special case $n=0$ (in which case $t$ becomes ${\bf 1}$), and the special case $n_i=0$ (in which case $\bigsqcup_{j=1}^{n_i} l_{ij}$ becomes ${\bf 0}$).
We refer to $c_i := \{l_{ij} \; | \; 1 \leq j \leq {n_i}\}$ as an \emph{(inner) clause} of $t$, and to $l_{ij}$ as an \emph{(inner) literal} of $c_i$.
We say that a set of inner clauses is \emph{satisfiable}
if there exists an assignment from 
$V \rightarrow {\cal P}({\mathbb N})$
such that for all inner clauses, the union of the evaluation of all literals equals $\mathbb N$ (this is the case if and only if the formula $t={\bf 1}$ has a satisfying assignment). 

We assume that all quantifier-free first-order formulas $\phi$ over the signature $\{\sqcap,\sqcup,c,{\bf 0},{\bf 1}\}$
are written in \emph{(outer) conjunctive normal form (CNF)}, i.e., as $\phi = \bigwedge_{i=1}^m \bigvee_{j=1}^{m_i} L_{ij}$
where $L_{ij}$ is either of the form $t = {\bf 1}$ (a \emph{positive (outer) literal}) or of the form
$t \neq {\bf 1}$ (a \emph{negative (outer) literal}). 
Again, it is well-known and easy to see that we can for every quantifier-free formula find a formula in this form which is equivalent to it in every Boolean algebra.
We refer to $C_i := \{L_{ij} \; | \; 1 \leq j \leq {m_i} \}$ as 
an \emph{(outer) clause} of $\phi$, and to $L_{ij}$ as an \emph{(outer) literal} of $C_i$.
Whenever convenient, we identify $\phi$ with its set of clauses.

\section{${\mathcal EI}$ Set Constraints}
\label{sect:ei}
To define $\mathcal{EI}$ set constraints, we need to introduce
a series of important functions defined on the set of subsets of natural numbers.

\begin{defn}
Let 
\begin{itemize}
\item $i \colon ({\cal P}({\mathbb N}))^2 \rightarrow {\cal P}({\mathbb N})$ 
be the function that maps $(S_1, S_2)$ to the set
$\{2x \; | \; x \in S_1\} \cup \{2x + 1 \; | \; x \in S_2\}$;
\item $F$ be the function that maps $S \subseteq {\mathbb N}$
to the set of finite non-empty subsets of $S$;
\item $G \colon {\mathbb N} \rightarrow F({\mathbb N})$ be a bijection between $\mathbb N$ and the set of finite non-empty subsets of $\mathbb N$ (since both sets are countable, such a bijection exists); 
\item $e \colon  {\cal P}({\mathbb N}) \rightarrow {\cal P}({\mathbb N})$ be 
defined by $$e(S) = \{G^{-1}(T) \; | \; T \in F(S) \} \, ;$$
\item $ei$ be the function defined by $ei(x,y) \mapsto e(i(x,y))$.
\end{itemize}
\end{defn}

\begin{defn}
Let $f \colon ({\cal P}({\mathbb N}))^k \rightarrow {\cal P}({\mathbb N})$ be a function, and 
$R \subseteq {\cal P}({\mathbb N})^l$ be a relation. Then we say that $f$ \emph{preserves} $R$
if the following holds: for all $a^1,\dots,a^k \in ({\cal P}({\mathbb N}))^l$ we have that $(f(a_1^1,\dots,a_1^k),\dots,f(a_l^1,\dots,a_l^k)) \in R$ if $a^i \in R$ for all $i \leq k$. If $f$ does not preserve $R$, we also say that
$f$ \emph{violates} $R$. We say that $f$ \emph{strongly preserves} $R$
if for all $a^1,\dots,a^k \in ({\cal P}({\mathbb N}))^l$ we have that $(f(a_1^1,\dots,a_1^k),\dots,f(a_l^1,\dots,a_l^k)) \in R$ if and only if $a^i \in R$ for all $i \leq k$.
If $\phi$ is a first-order formula that defines a relation $R$ over $\mathfrak S$, and $f$ preserves (strongly preserves) $R$, then we also say that $f$ \emph{preserves (strongly preserves)} $\phi$.
Finally, if $g \colon ({\cal P}({\mathbb N}))^l \rightarrow {\cal P}({\mathbb N})$ is
a function, we say that $f$ \emph{preserves (strongly preserves)} $g$ if 
it preserves (strongly preserves) the graph of $g$, i.e., the relation $\big\{(x_1,\dots,x_l,g(x_1,\dots,x_l)) \; | \; x_1,\dots,x_l \subseteq {\mathbb N} \big\}$.
\end{defn}

Note that if an injective function $f$ preserves a function $g$,
then it must also strongly preserve $g$. 


\begin{defn}
The set of all relations with a quantifier-free first-order definition over $\mathfrak S$ that are preserved by the operation $ei$ is denoted by
$\mathcal{EI}$.
\end{defn}

{\bf Remark.} We will see later (Proposition~\ref{prop:e-atomless-correct} and Proposition~\ref{prop:ei-atomless-correct}) that the class $\mathcal{EI}$ is independent from the precise
choice of the operations $e$ and $i$.

\vspace{.2cm}

Proposition~\ref{prop:syntax} shows that $\mathcal{EI}$ has a large subclass, called Horn-Horn, which has an intuitive syntactic description. In Section~\ref{sect:horn-horn}
we also present many examples of relations that are from $\mathcal{EI}$ and of relations that are not from $\mathcal{EI}$.
Before, we will establish some properties of the functions $i$ and $e$.

\begin{fact}\label{fact:nat-iso}
The mapping $i$ is an isomorphism between $\mathfrak S^2$ and $\mathfrak S$.
\end{fact}
\begin{proof}
The mapping $i$ can be inverted by the mapping that sends
$S \subseteq \mathbb{N}$ to $\big(\{x \; | \; 2x \in S\}, \{x \; | \; 2x + 1 \in S\}\big)$. It is straightforward 
to verify that $i$ 
strongly preserves ${\bf 0}$, ${\bf 1}$, $c$, $\sqcup$, $\sqcap$.
\ignore{
\begin{itemize}
\item Clearly, $i(x, y) = \emptyset$ if and only if $x=y=\emptyset$.
\item Similarly, since the natural numbers are partitioned by the even and odd numbers, $i(x,y) = \mathbb{N}$ if and only if $x=y={\mathbb N}$.
\item Let $S_1$ and $S_2$ be subsets of $\mathbb{N}$. To verify that 
$i$ preserves $c$ we have to show that $i(c(S_1, S_2))$, which is by definition
equal to $i(\overline{S_1}, \overline{S_2})$, equals $c(i(S_1, S_2))$. Suppose that $x = 2x'$. Then:
\begin{eqnarray*}
x \in i(\overline{S_1}, \overline{S_2}) & \Leftrightarrow & x' \in \overline{S_1} \\
& \Leftrightarrow & x' \notin S_1 \\
& \Leftrightarrow & 2x' \notin i(S_1, S_2) \\
& \Leftrightarrow & x \in \overline{i(S_1, S_2)}
\end{eqnarray*}
The argument for $x = 2x' + 1$ is analogous. Thus, $i$ preserves $c$. Since $i$ is injective, it even strongly preserves $c$.

\item Let $(S_1, S_2)$ and $(T_1, T_2)$ be from $({\cal P}(\mathbb{N})^2$. We have to show that $i((S_1, S_2) \sqcup (T_1, T_2))$, which is by
definition equal to $i(S_1 \sqcup T_1, S_2 \sqcup T_2)$, equals $i(S_1, S_2) \sqcup i(T_1, T_2)$. With $x = 2x'$ as before:
\begin{align*}
x \in i(S_1 \sqcup T_1, S_2 \sqcup T_2) & \Leftrightarrow x' \in (S_1 \sqcup T_1) \\
& \Leftrightarrow 2x' \in i(S_1, S_2) \sqcup i(T_1, T_2) \\
\end{align*}
The argument for $x = 2x' + 1$ is again analogous.
Thus, $i$ preserves $\sqcup$. 
Since $i$ is injective, it even strongly preserves $\sqcup$.
\end{itemize}
The verification for $\sqcap$ is similar to that for $\sqcup$.}
\end{proof}

We write $x \sqsubseteq y$ as an abbreviation 
for $x \sqcap y = x$.
\begin{prop}\label{prop:nat-core}
The function $e$ has the following properties.
\begin{itemize} 
\item $e$ is injective, 
\item $e$ strongly preserves ${\bf 1}$, ${\bf 0}$, and $\sqcap$,  and
\item 
for $x,y,z \in {\cal P}({\mathbb N})$ such that $x \sqcup y = z$, not $x \sqsubseteq y$, and not $y \sqsubseteq x$, we have that $e(x) \sqcup e(y) \sqsubsetneqq e(z)$. 
\end{itemize}
\end{prop}
\begin{proof}
We verify the properties one by one.
Since $G$ is bijective, $e(x)=e(y)$ if and only if $x$
and $y$ have the same finite subsets. This is the case if and only if 
$x=y$, and hence $e$ is injective. 
Thus, to prove that $e$ \emph{strongly} preserves ${\bf 1}$, ${\bf 0}$, and $\sqcap$, it suffices to check that $e$
preserves ${\bf 1}$, ${\bf 0}$, and $\sqcap$.

Since $G$ is bijective, we have that 
$G({\mathbb N})$ equals the set of all finite subsets of $\mathbb N$,
and hence $e({\mathbb N})={\mathbb N}$, which shows that $e$ preserves ${\bf 1}$.
We also compute $e(\emptyset) = G^{-1}(F(\emptyset))=G^{-1}(\emptyset) = \emptyset$.

Next, we verify that for all $x,y \in {\cal P}({\mathbb N})$
we have $e(x) \sqcap e(y) = e(x \sqcap y)$.
Let $a \in {\mathbb N}$ be arbitrary. We have $a \in e(x) \sqcap e(y)$ if and only if $G(a) \in F(x) \cap F(y)$. By definition of $F$ and since $G(a)$ is a finite subset of ${\mathbb N}$, this is the case if and only if $G(a) \in F(x \sqcap y)$. This is the case
if and only if $a \in e(x \sqcap y)$, which concludes the proof that $e$ preserves $\sqcap$. 

We verify that if $x \sqcup y = z$, not $x \sqsubseteq y$, and not $y \sqsubseteq x$, then $e(x) \sqcup e(y) \sqsubsetneqq e(z)$. 
First observe that for all $u,v \subseteq {\mathbb N}$ with 
$u \sqsubseteq v$ we have $e(u) \sqsubseteq e(v)$ since $e$ preserves $\sqcap$. This implies 
that $e(x) \sqcup e(y) \sqsubseteq e(z)$.
Since $x \not\sqsubseteq y$ and $y \not\sqsubseteq x$, there are $a,b$ such that $a \in x$, $a \notin y$, $b \in y$, $b \notin x$.
Then we have that $\{a,b\} \in F(z)$, 
but $\{a,b\} \notin F(x) \cup F(y)$. Hence, $G^{-1}(\{a,b\}) \in e(z)$,
but $G^{-1}(\{a,b\}) \notin e(x) \sqcup e(y)$. This shows
that $e(z) \neq e(x) \sqcup e(y)$.  
\end{proof}

Note that in particular $e$ preserves $\sqsubset$, $\sqsubseteq$, and $||$. Moreover, $e(c(x)) \sqsubseteq c(e(x))$: this follows from preservation of $||$, since $x || c(x)$, and therefore $e(x) || e(c(x))$, which is equivalent to the inclusion above. 
Both $e$ and $i$ strongly preserve $\sqcap$, ${\bf 0}$, and ${\bf 1}$, and therefore also $ei$ strongly preserves $\sqcap$, ${\bf 0}$, and ${\bf 1}$. 

\section{Horn-Horn Set Constraints}
\label{sect:horn-horn}
A large and important subclass of ${\mathcal EI}$ set constraints
is the class of Horn-Horn set constraints.

\begin{defn}\label{def:horn-horn}
A quantifier-free first-order formula is called \emph{Horn-Horn}
if
\begin{enumerate}
\item every outer clause is \emph{outer Horn}, i.e., contains at most one positive outer literal, and
\item every inner clause of positive outer literals is 
\emph{inner Horn}, i.e., contains at most one positive inner literal.
\end{enumerate}
A relation $R \subseteq {\cal P}({\mathbb N})^k$
is called
\begin{itemize}
\item \emph{outer Horn} if it can be defined over 
$\mathfrak S$ by a conjunction of outer Horn clauses;
\item \emph{inner Horn} if it can be defined 
over $\mathfrak S$ by a formula of the form
$(c_1 \sqcap \cdots \sqcap c_k) = {\bf 1}$ where
each $c_i$ is inner Horn;
\item \emph{Horn-Horn} if it can be defined by a Horn-Horn formula over $\mathfrak S$.
\end{itemize} 
\end{defn}

The following is a direct consequence of the fact that 
isomorphisms between $\Gamma^k$ and $\Gamma$
preserve Horn formulas over $\Gamma$; since the simple proof is instructive for what follows, we give it here for the special case that is relevant here. 
\begin{fact}\label{fact:outer-horn}
Outer Horn relations are preserved by $i$. 
\end{fact}
\begin{proof}
Let $\phi$ be a conjunction of outer Horn clauses with variables $V$.
Let $\{t_0 = {\bf 1},t_1 \neq {\bf 1}, \dots, t_k \neq {\bf 1}\}$ 
be an outer clause of $\phi$.
Let $u,v \colon V \rightarrow {\cal P}({\mathbb N})$ 
be two assignments that
satisfy this clause. Let $w \colon V\rightarrow\cal{P}(\mathbb{N})$
be given by $x\mapsto i(u(x),v(x))$. Suppose that $w$ satisfies $t_j={\bf 1}$ for all $1\leq j\leq k$. 
Since $i$ is injective we must have that $t_j = \bf 1$ for both $u$ and $v$ for $1 \leq j \leq k$, and therefore neither assignment satisfies the negative literals.
Hence, $u$ and $v$ must satisfy $t_0 = {\bf 1}$.
Since $i$ is an isomorphism between ${\mathfrak S}^2$ and 
${\mathfrak S}$, it preserves in particular $t_0={\bf 1}$,  and hence $w$ also satisfies $t_0={\bf 1}$.
\end{proof}

\ignore{ Proof by contradiction
\begin{proof}
Let $\phi$ be a conjunction of outer Horn clauses with variables $V$.
Let $\{t_0 = {\bf 1},t_1 \neq {\bf 1}, \dots, t_k \neq {\bf 1}\}$ 
be an outer clause of $\phi$.
Let $u,v \colon V \rightarrow {\cal P}({\mathbb N})$ 
be two assignments that satisfy this clause.
Suppose for contradiction that $w \colon V \rightarrow {\cal P}({\mathbb N})$ defined by $x \mapsto i(u(x),v(x))$ does not satisfy this clause.
In particular, $w$ satisfies $t_j={\bf 1}$ 
for all $1 \leq j \leq k$.
Since $i$ is injective we must have that $t_j = \bf 1$ for both $u$ and $v$ for $1 \leq j \leq k$, and therefore neither assignment satisfies the negative literals.
Hence, $u$ and $v$ must satisfy $t_0 = {\bf 1}$.
Since $i$ is an isomorphism between ${\mathfrak S}^2$ and 
${\mathfrak S}$, it preserves in particular $t_0={\bf 1}$, and hence
$w$ also satisfies $t_0={\bf 1}$, a contradiction.
\end{proof}
}

\begin{prop}\label{prop:inner-horn}
Inner Horn relations are strongly preserved by $e$. 
\end{prop}
\begin{proof}
Observe that $x \sqcup ( \bigsqcup_j \overline{y}_j) = {\bf 1}$ is equivalent to $x \sqcap ( \bigsqcap_j y_j) = \bigsqcap_j y_j$, which is
strongly preserved by $e$ since $e$ strongly preserves $\sqcap$. This clearly implies
the statement. 
\end{proof}

Note that Fact~\ref{fact:outer-horn} and Proposition~\ref{prop:inner-horn} imply
that $ei$ strongly preserves inner Horn relations.
We later also need the following. 


\begin{lem}\label{lem:e}
Let $x_1,\dots,x_k,y_1,\dots,y_l \subseteq {\mathbb N}$, where $k\geq1$. Then the
following are equivalent.
\begin{enumerate}
\item $e(x_1) \sqcup \dots \sqcup e(x_k) \sqcup \overline{e(y_1)} \sqcup \dots \sqcup \overline{e(y_l)} = {\bf 1}$.
\item there exists an $i \leq k$ such that $x_i \sqcup ( \bigsqcup_j \overline{y}_j) = {\bf 1}$.
\item there exists an $i \leq k$ such that $e(x_i) \sqcup ( \bigsqcup_j \overline{e(y_j)}) = {\bf 1}$.
\end{enumerate}
For $k=0$, we have that $\bigsqcup_j \overline{y}_{j \leq l} = {\bf 1}$ if and only if
$\bigsqcup_{j \leq l} \overline{e(y_j)} = {\bf 1}$.
\end{lem}
\begin{proof}
For the implication from $(1)$ to $(2)$, 
suppose that there is for every $i \leq k$ an $a_i \in {\mathbb N}$ such that $a_i \notin X_i := x_i \sqcup ( \bigsqcup_j \overline{y}_j)$. Let $c$ be $G^{-1}\big(\{a_1,a_2,\dots,a_k\}\big)$. 
Then for each $i \leq k$, we have that $c \notin 
e(x_i) \sqcup \bigsqcup_{j \leq l} \overline{e(y_j)}$.
To see this, first observe that $a_i \in \bigsqcap_{j\leq l} y_j \sqcap \overline{x}_i$. Therefore, $\{a_1,\dots,a_k\} \in \bigsqcap_{j \leq l} F(y_j) \sqcap \overline{F(x_i)}$ for all $i \leq k$.
We conclude that 
$c \notin e(x_1) \sqcup \dots \sqcup e(x_k) \sqcup \overline{e(y_1)} \sqcup \dots \sqcup \overline{e(y_l)}$. 

The implication $(2) \Rightarrow (3)$ follows directly from Proposition~\ref{prop:inner-horn}.
The implication $(3) \Rightarrow (1)$ is trivial. 
The second statement is a direct consequence of 
Proposition~\ref{prop:inner-horn}.
\end{proof}

\begin{prop}\label{prop:syntax}
Every Horn-Horn relation is preserved by $e$ and $i$; in particular, 
it is from $\mathcal EI$.
\end{prop}
\begin{proof}
Suppose that $R$ has a Horn-Horn definition $\phi$ 
over $\mathfrak S$ with variables $V$.
Since $R$ is in particular outer Horn, 
it is preserved by $i$ by Fact~\ref{fact:outer-horn}.

Now we verify that $R$ is preserved by $e$.
Let $u \colon V \rightarrow {\cal P}({\mathbb N})$ be an assignment
that satisfies $\phi$. That is, $u$ satisfies 
at least one literal in each outer clause of $\phi$.
It suffices to show 
that the assignment $v \colon V \rightarrow {\cal P}({\mathbb N})$
defined by $x \mapsto e(u(x))$ satisfies the same outer literal.
Suppose first that the outer literal is 
positive; because $\phi$ is Horn-Horn, it is of the form
$x \sqcup \overline{y_1} \sqcup \dots \sqcup \overline{y_l} = {\bf 1}$ or of the form $\overline{y_1} \sqcup \dots \sqcup \overline{y_l} = {\bf 1}$, which is preserved by $e$ by Lemma~\ref{lem:e}. 


Now, suppose that the outer literal is negative, that is, of the form 
$x_1 \sqcup \dots \sqcup x_k \sqcup \overline{y_1} \sqcup \dots \sqcup \overline{y_l} \neq {\bf 1}$ for some $k \geq 0$. We will treat the case $k \geq 1$, the other case being similar.
Suppose for contradiction that 
$v(x_1) \sqcup \dots \sqcup v(x_k) \sqcup \overline{v(y_1)} \sqcup \dots \sqcup \overline{v(y_l)} = {\bf 1}$.
By Lemma~\ref{lem:e}, there exists an $i \leq k$ such that $u(x_i) \sqcup ( \bigsqcup_j u(\overline{y}_j)) = {\bf 1}$.
But then we have in particular that 
$u(x_1) \sqcup \dots \sqcup u(x_k) \sqcup u(\overline{y_1}) \sqcup \dots \sqcup u(\overline{y_l}) = {\bf 1}$, in contradiction
to the assumption that $u$ satisfies $\phi$. 
\end{proof}

\noindent {\bf Examples.}
\begin{enumerate}
\item The disjointness relation $||$ is Horn-Horn:
it has the definition $\bar x \sqcup \bar y = {\bf 1}$. 
\item The inequality relation $\neq$ is inner Horn:
it has the definition $(y \sqcup \bar x) \sqcap (x \sqcup \bar y) \neq {\bf 1}$.
\item Using the previous example, the relation $\{(x,y,u,v) \; | \; x \neq y \vee u = v\}$ 
can easily be seen to be Horn-Horn. 
\item The ternary 
relation $\{(x,y,z) \; | \; x \cap y \subseteq z\}$, which we have
encountered above, has the Horn-Horn definition  
$\bar x \sqcup \bar y \sqcup z = {\bf 1}$. 
\item
Examples of relations that are clearly \emph{not} Horn-Horn: $\{ (x,y) \; | \; x \sqcup y = {\bf 1} \}$ is violated by $e$, 
and $\{ (x,y,z) \; | \; (x=y) \vee (y=z)\}$ is violated by $i$.
\item The formula 
\begin{align*}
& (x \sqcap y \neq x) \\
\wedge \; & (x \sqcap y \neq y) \\ 
\wedge \; & (v={\bf 1} \; \vee \; u={\bf 1} \; \vee \;  x \sqcup y \neq {\bf 1})
\end{align*}
is clearly not Horn-Horn.
However, the relation defined by the formula is from $\cal EI$:
if $(x_1,y_1,u_1,u_2)$ und $(x_2,y_2,u_2,v_2)$ are from that relation, then neither $i(x_1,x_2) \sqsubseteq i(y_1,y_2)$ 
nor $i(y_1,y_2) \sqsubseteq i(x_1,x_2)$.
By Proposition~\ref{prop:nat-core}, 
$(ei(x_1,x_2),ei(y_1,y_1),ei(u_1,u_2),ei(v_1,v_2))$ 
satisfies the formula. 

There is no equivalent Horn-Horn 
formula, since the formula is not preserved by $i$.
\item
The formula $((x \sqcup y \neq {\bf 1}) \vee (u \sqcup v = {\bf 1}))\wedge (\bar x \sqcup y \neq {\bf 1}) \wedge (x \sqcup \bar y \neq {\bf 1})$ is not Horn-Horn. However, it is preserved by $e$ and by $i$:
the reason is that one of its clauses has the negative literal $x \sqcup y \neq {\bf 1}$, and the conjuncts
$\{\bar x \sqcup y \neq {\bf 1}\}$ and $\{x \sqcup \bar y \neq {\bf 1}\}$.
Therefore, for every tuple $t \in R$ the tuple $e(t)$ satisfies $x \sqcup y \neq {\bf 1}$ and is in $R$ as well. By Fact~\ref{fact:outer-horn}, $R$ is preserved by $i$. 

In this case, the authors suspect that there is no equivalent 
Horn-Horn formula. More generally, it is an open problem whether there exist formulas that are preserved by $e$ and $i$, but that are \emph{not} equivalent to a Horn-Horn formula.
\end{enumerate} 

\begin{prop}
Drakengren and Jonsson's set constraint language only contains
 Horn-Horn relations.
\end{prop}
\begin{proof}
For inclusion $x \subseteq y$, disjointness $||$, and inequality $\neq$
this has been discussed in the examples. Horn-Horn is preserved
under adding additional outer 
disequality literals to the outer clauses, so all relations considered
in Drakengren and Jonsson's language are Horn-Horn. 
\end{proof}

We prepare now some results that can be viewed
as a partial converse of Proposition~\ref{prop:syntax}.

\begin{defn}
A quantifier-free first-order formula $\phi$ (in the syntactic form described 
at the end of Section~\ref{sect:setcsps}) is called \emph{reduced}
if if every formula
obtained from $\phi$ by removing an outer literal
is not equivalent to $\phi$ over $\mathfrak S$.
\end{defn}

\begin{lem}
\label{lem:reduced}
Every quantifier-free formula is over $\mathfrak S$ equivalent to
a reduced formula.
\end{lem}
\begin{proof}
It is clear that every quantifier-free formula can be written as a formula $\phi$ in CNF and in the form as we have discussed it after Theorem~\ref{thm:np}.
We now remove successively outer literals 
as long as this results in an equivalent formula. 
\end{proof}

We first prove the converse of Fact~\ref{fact:outer-horn}.

\begin{prop}\label{prop:outer-horn}
Let $\phi$ be a reduced formula that is preserved by $i$. 
Then each outer clause of $\phi$ is Horn.
\end{prop}
\begin{proof}
Let $V$ be the set of variables of $\phi$.
Assume for contradiction that $\phi$ contains an outer clause with two
positive literals, $t_1 = {\bf 1}$ and $t_2 = {\bf 1}$. 
If we remove
the literal $t_1 = {\bf 1}$ from its clause $C$, the resulting formula
is inequivalent to $\phi$, and hence there is an assignment $s_1 \colon V \rightarrow {\cal P}({\mathbb N})$ that
satisfies none of the literals of $C$ except for $t_1= {\bf 1}$.
Similarly, there is an assignment $s_2 \colon V \rightarrow {\cal P}({\mathbb N})$ that
satisfies none of the literals of $C$ except for $t_2 = {\bf 1}$.
By injectivity of $i$, and since $i$ strongly preserves $c, \sqcap, \sqcup$, and ${\bf 1}$, the assignment $s \colon V \rightarrow {\cal P}({\mathbb N})$
defined by $x \mapsto i(s_1(x),s_2(x))$
 does not satisfy the two
literals $t_1 = {\bf 1}$ and $t_2 = {\bf 1}$.
Since $i$ strongly preserves $c$, $\sqcup$, $\sqcap$,
none of the other 
literals in $C$ is satisfied by those mappings as well, in contradiction to the assumption that $\phi$ is preserved by $i$.
\end{proof}

\begin{defn}
Let $V$ be a set of variables, and 
$s \colon V \rightarrow {\cal P}({\mathbb N})$ be a mapping. Then 
a function from $V \rightarrow {\cal P}({\mathbb N})$
of the form $x \mapsto e(s(x))$ is called
a \emph{core assignment}. 
\end{defn}

\begin{lem}\label{lem:strongly-reduced}
For every quantifier-free formula $\phi$ there exists a formula $\psi$ such that all inner clauses 
are inner Horn, and such that
$\phi$ and $\psi$ have the same satisfying core assignments.
If $\phi$ is preserved by $ei$, then the set of all satisfying core assignments of $\psi$ is closed under $ei$.
\end{lem}
\begin{proof}
Suppose that $\phi$ has an outer clause $C$ with
a positive outer literal 
$t = {\bf 1}$ such that $t$ contains an inner clause
$c := x_1 \sqcup \cdots \sqcup x_k \sqcup \overline y_1 \sqcup \cdots \sqcup \overline y_l$ that is not
Horn, i.e., $k \geq 2$. Then we replace 
the outer literal $t = {\bf 1}$ in $\phi$ by $k$ literals
$t_1 = {\bf 1}, \dots, t_k = {\bf 1}$
where $t_i$ is obtained from $t$ by 
replacing $c$ by $x_i \sqcup \overline y_1 \sqcup \cdots \sqcup \overline y_l$. 

We claim that the resulting formula $\phi'$ 
has the same set of 
satisfying core assignments. 
Observe that
$x_i \sqcup \overline y_1 \sqcup \cdots \sqcup \overline y_l \sqsubseteq c$, and hence
$t_i = {\bf 1}$ implies $t = {\bf 1}$.
An arbitrary satisfying assignment of 
$\phi'$ satisfies either one of the positive outer 
literals $t_i = {\bf 1}$,
in which case that observation shows that it also satisfies $\phi$, or it satisfies one of the other 
outer literals of $C$, in which
case it also satisfies this literal in $\phi$.
Hence, $\phi'$ implies $\phi$. Conversely, 
let $s$ be a satisfying core assignment of $\phi$. 
If $s$ satisfies a literal from $C$ other than $t = {\bf 1}$,
then it also satisfies this literal in $\phi'$, and 
$s$ satisfies $\phi'$. Otherwise, 
$s$ must satisfy $t = {\bf 1}$, and hence
$s(x_1) \sqcup \cdots \sqcup s(x_k) \sqcup \overline{s(y_1)} \sqcup \cdots \sqcup \overline{s(y_l)} = {\bf 1}$.
Since $s$ is a core assignment, Lemma~\ref{lem:e}
implies that there exists an $i \leq k$ such that
$s(x_i) \sqcup \overline{s(y_1)} \sqcup \cdots \sqcup \overline{s(y_l)} = {\bf 1}$. So $s$ satisfies $\phi'$. 

Suppose that $\phi$ has an outer clause $C$ with
a negative outer literal 
$t \neq {\bf 1}$ such that $t$ contains an inner clause
$c := x_1 \sqcup \cdots \sqcup x_k \sqcup \overline y_1 \sqcup \cdots \sqcup \overline y_l$ that is not
Horn, i.e., $k \geq 2$. Then we replace the clause
$C$ in $\phi$ by $k$ clauses 
$C_1$, \dots, $C_k$
where $C_k$ is obtained from $C$ by 
replacing $c$ with $x_i \sqcup \overline y_1 \sqcup \cdots \sqcup \overline y_l$. 

We claim that the resulting formula $\phi'$ 
has the same set of 
satisfying core assignments. 
Observe that $x_1 \sqcup \cdots \sqcup x_k \sqcup \overline y_1 \sqcup \cdots \sqcup \overline y_l \neq {\bf 1}$ implies that $x_i \sqcup \overline y_1 \sqcup \cdots \sqcup \overline y_l \neq {\bf 1}$, for
every $i \leq k$.
The observation shows that an arbitrary assignment
of $\phi$ is also an assignment of $\phi'$. 
Conversely, let $s$ be a satisfying core assignment of $\phi'$. If $s$ satisfies one of the other literals 
of $C$ other than $t \neq {\bf 1}$,
then 
$s$ satisfies $\phi$. Otherwise, $s$ must satisfy
$x_i \sqcup \overline y_1 \sqcup \cdots \sqcup \overline y_l \neq {\bf 1}$ for all $i \leq k$,
and by Lemma~\ref{lem:e} 
we have that $s$ also satisfies 
$x_1 \sqcup \cdots \sqcup x_k \sqcup \overline y_1 \sqcup \cdots \sqcup \overline y_l \neq {\bf 1}$.

We perform these
replacements until we obtain a formula $\phi'$
where all inner clauses are 
Horn; this formula satisfies the requirements
of the first statement of the lemma. 

To prove the second statement, 
let $u,v \colon V \rightarrow {\cal P}({\mathbb N})$ be two satisfying core assignments of $\phi'$. Since
$\phi'$ and $\phi$ have the same satisfying 
core assignments, $u$ and $v$ also satisfy $\phi$.
Then the mapping $w \colon V \rightarrow {\cal P}({\mathbb N})$ given by
$x \mapsto ei(u(x),v(x))$ is a core assignment, and because $ei$ preserves $\phi$, the mapping $w$ satisfies $\phi$. Since $\phi$ and $\phi'$ have the same core assignments, $w$ is also a satisfying assignment of $\phi'$, which proves the statement. 
\end{proof}

\begin{defn}
A quantifier-free first-order formula $\phi$ (in the syntactic form described 
at the end of Section~\ref{sect:setcsps}) is called \emph{strongly reduced}
if every formula
obtained from $\phi$ by removing an outer literal
does not have the same set of satisfying core assignments
 over $\mathfrak S$.
\end{defn}

\begin{prop}\label{prop:core-horn-horn}
Let $\phi$ be a strongly reduced formula all of whose inner clauses are Horn. If the set of satisfying core assignments
of $\phi$ is closed under $ei$, then $\phi$ is Horn-Horn.
\end{prop}
\begin{proof}
Let $V$ be the set of variables of $\phi$.
It suffices to show that all clauses of $\phi$ are outer Horn. 
Assume for contradiction
that $\phi$ contains an outer clause with two positive literals, $t_1 = {\bf 1}$ and $t_2 = {\bf 1}$. If we remove the literal $t_1 = {\bf 1}$
from its clause $C$, the resulting formula has strictly less satisfying core assignments; this shows the existence of a core assignment
$s_1 \colon V \to {\cal P}({\mathbb N})$ that satisfies
none of the literals of $C$ except for $t_1 = {\bf 1}$. 
Similarly, there exists a core assignment $s_2 \colon V \to {\cal P}({\mathbb N})$ that satisfies none of the literals of $C$
except for $t_2 = {\bf 1}$. By assumption, the
inner clauses of $t_1$ and $t_2$ are Horn. 
We claim that the assignment $s \colon V \rightarrow {\cal P}({\mathbb N})$ defined by $x \mapsto ei(s_1(x),s_2(x))$ does not satisfy the clause $C$.
Since $ei$ strongly preserves inner Horn clauses,
we have that $s$ does not satisfy $t_1 = {\bf 1} \vee t_2 = {\bf 1}$. For the same reasons $s$ does not
satisfy any other literals in $C$; this contradicts the assumption that the satisfying core assignments for $\phi$ are preserved by $ei$.
\end{proof}

\begin{prop}\label{prop:core-reduction}
Let $\Gamma$ be a finite set constraint language from $\mathcal EI$. 
Then $\Csp(\Gamma)$ can be reduced in linear time 
to the problem to find a satisfying assignment 
for a given set of Horn-Horn clauses.
\end{prop}
\begin{proof}
Let $\Phi$ be an instance of $\Csp(\Gamma)$, 
and let $V$ be
the set of variables that appear in $\Phi$.
For each constraint $R(x_1,\dots,x_k)$ from $\Phi$, let 
$\phi_R$ be the definition of $R$ over $\mathfrak S$.
By Lemma~\ref{lem:strongly-reduced}, there
exists a formula $\psi_R$ that has the same 
satisfying core assignments as $\phi_R$ and where
all inner clauses are Horn; moreover, since
$\phi_R$ is preserved by $ei$, the lemma
asserts that the set of all satisfying core assignments
of $\psi_R$ is preserved by $ei$.
We can assume without loss of generality that
$\psi_R$ is strongly reduced; this can be seen
similarly to Lemma~\ref{lem:reduced}.
By Proposition~\ref{prop:core-horn-horn}, the formula $\psi_R$ is Horn-Horn. 

Let $\Psi$ be the set of all Horn-Horn clauses of formulas $\psi_R(x_1,\dots,x_k)$ obtained from constraints $R(x_1,\dots,x_k)$ in $\Phi$ in the described manner. We claim that $\Phi$ is a satisfiable instance of $\Csp(\Gamma)$ if and only if $\Psi$ is satisfiable. This follows from the fact that
for each constraint $R(x_1,\dots,x_k)$ in $\Phi$,
the formulas $\phi_R$ and $\psi_R$ have the same
satisfying core assignments, and that
both $\phi_R$ and $\psi_R$ 
are preserved by $ei$
(for $\psi_R$ this follows from Proposition~\ref{prop:syntax}), so in particular by the function
$x \mapsto ei(x,x)$.
\end{proof}

Note that in Proposition~\ref{prop:core-reduction}
we reduce satisfiability for $\mathcal EI$
to satisfiability for a proper subclass of 
Horn-Horn set constraints: while for general
Horn-Horn set constraints we allow 
that inner clauses of negative outer literals are
not Horn, the reduction only produces 
Horn-Horn clauses where \emph{all} inner clauses are Horn. 

\section{Algorithm for Horn-Horn Set Constraints}
\label{sect:algorithm}
We present an algorithm that takes as input a set $\Phi$ of Horn-Horn
clauses and decides satisfiability of $\Phi$ over
$\mathfrak S = ({\cal P}(\mathbb{N}); \sqcup, \sqcap, c, {{\bf 0}}, {\bf 1})$ in time quadratic to the length of the input. 
By Proposition~\ref{prop:core-reduction}, this section will therefore conclude the proof that $\Csp(\Gamma)$ is tractable when all relations
in $\Gamma$ are from $\mathcal EI$.

We first discuss an important sub-routine of our algorithm, which we call the \emph{inner resolution algorithm}.
As in the case of Boolean positive unit resolution~\cite{horn-linear}
one can implement the procedure Inner-Res such that it runs in linear time in the input size. 

\begin{figure}[h]
\begin{center}
\small
\fbox{
\begin{tabular}{l}
Inner-Res($\Phi$) \\
{\rm // Input: A finite set $\Phi$ of inner Horn clauses} \\
{\rm // Accepts iff $\bigsqcap \Phi ={\bf 1}$ is satisfiable} \\
During the entire algorithm: \\
\hspace{.5cm} if $\Phi$ contains an empty clause, then  
reject. \\
Repeat := true \\
While Repeat = true do \\
\hspace{.5cm}  Repeat := false \\
\hspace{.5cm}        If $\Phi$ contains a positive unit clause $\{x\}$ then \\
\hspace{1cm} Repeat := true \\
\hspace{1cm} Remove all clauses where the literal $x$ occurs. \\
\hspace{1cm} Remove the literal $\overline{x}$ from all clauses. \\
\hspace{.5cm} End if \\
Loop \\
Accept
\end{tabular}}
\end{center}
\caption{Inner Resolution Algorithm.}
\label{fig:inner}
\end{figure}

\begin{lem}\label{lem:pre-inner-res}
Let $\Phi$ be a finite set of inner Horn clauses. Then the following are equivalent.
\begin{enumerate}
\item $\bigsqcap \Phi = {\bf 1}$ is satisfiable over $\mathfrak S$.
\item Inner-Res$(\Phi)$ from Figure~\ref{fig:inner} accepts.
\item $\bigsqcap \Phi = {\bf 1}$ has a solution whose image is contained in $\{\emptyset,\mathbb{N}\}$.
\end{enumerate}
\end{lem}
\begin{proof}
It is obvious that $\bigsqcap \Phi = {\bf 1}$ is unsatisfiable when Inner-Res$(\Phi)$ rejects; in fact, for all inner clauses $c$ derived by Inner-Res from $\Phi$, the formula $c = {\bf 1}$ is logically implied by $\bigsqcap \Phi = {\bf 1}$.
Conversely, if the
algorithm accepts then we can set all eliminated variables to $\mathbb{N}$ and all remaining variables
to $\emptyset$, which satisfies all clauses: in the removed clauses the positive literal is satisfied, and
in the remaining clauses we have at least one negative literal at the final stage of the
algorithm, and all clauses with negative literals at the final stage of the algorithm are
satisfied. 
\end{proof}

The proof of the previous lemma shows that
$\bigsqcap \Phi = {\bf 1}$ is satisfiable over $\mathfrak S$ if and only if $\bigsqcap \Phi = {\bf 1}$ is satisfiable over the two-element Boolean algebra. As we will see in the following, this holds more generally (and not only for inner Horn clauses). 
The following should be well-known, and can be shown with the same
proof as given in~\cite{KoppelbergBoolenAlgebras} for the weaker Proposition 2.19 there. We repeat the proof here for the convenience of the reader (for definitions of the notions appearing in the proof, however, we refer to~\cite{KoppelbergBoolenAlgebras}). 

\begin{fact}
\label{fact:ba}
Let $t_1,t_2$ be terms over $\{\sqcap,\sqcup,c,{\bf 0},{\bf 1}\}$. Then the following are equivalent: 
\begin{enumerate}
\item $t_1 = {\bf 1} \wedge t_2 \neq {\bf 1}$ is satisfiable over the two-element Boolean algebra;
\item $t_1 = {\bf 1} \wedge t_2 \neq {\bf 1}$ is satisfiable over \emph{all} Boolean algebras;
\item $t_1 = {\bf 1} \wedge t_2 \neq {\bf 1}$ is satisfiable in a Boolean algebra.
\end{enumerate}
\end{fact}
\begin{proof}
Obviously, 1 implies 2, and 2 implies 3. 
For 3 implies 1, assume that $t_1 = {\bf 1} \wedge t_2 \neq {\bf 1}$ has a satisfying assignment in some Boolean algebra $\mathfrak C$. 
Let $c$ be the element denoted by $t_2$ in 
$\mathfrak C$ under this assignment. 
It is well-known that every element $a \neq {\bf 0}$ of a Boolean algebra is contained in an ultrafilter (see e.g.~Corollary 2.17 in~\cite{KoppelbergBoolenAlgebras}). So let $\cal U$ be an ultrafilter of $\mathfrak C$ that contains $\overline{c}$,
and let $f \colon {\mathfrak C} \to \{{\bf 0},{\bf 1}\}$ be the characteristic function of $\cal U$.
Then $f$ is a homomorphism from $\mathfrak C$ to the two-element Boolean algebra that maps $c$ to ${\bf 0}$; thus
$t_1 = {\bf 1} \wedge t_2 \neq {\bf 1}$ is satisfiable 
over $\{{\bf 0},{\bf 1}\}$.
\end{proof}

The same statement for $t_1 = {\bf 1}$ instead of
$t_1 = {\bf 1} \wedge t_2 \neq {\bf 1}$ is Proposition 2.19. in~\cite{KoppelbergBoolenAlgebras}.
Fact~\ref{fact:ba} has the following consequence that is crucial for the way how we use the inner resolution
procedure in our algorithm. 

\begin{lem}\label{lem:inner-res}
Let $\Psi$ be a finite set of inner Horn clauses. Then Inner-Res$(\Psi {\cup} \{
\overline{x}_1,\dots,\overline{x}_k,y_0,\dots,y_l\})$
rejects
if and only if $\bigsqcap \Psi = {\bf 1}$ implies that $x_1 \sqcup \dots \sqcup x_k \sqcup \overline{y}_1 \sqcup \dots \sqcup \overline{y}_l = {\bf 1}$ over $\mathfrak S$. 
\end{lem}
\begin{proof}
$\bigsqcap \Psi = {\bf 1}$ implies that $x_1 \sqcup \dots \sqcup x_k \sqcup \overline{y}_1 \sqcup \dots \sqcup \overline{y}_l = {\bf 1}$ if and only if
$\bigsqcap \Psi = {\bf 1} \wedge x_1 \sqcup \dots \sqcup x_k \sqcup \overline{y}_1 \sqcup \dots \sqcup \overline{y}_l \neq {\bf 1}$ is unsatisfiable over $\mathfrak S$. 
By Fact~\ref{fact:ba}, this is the case if and only if
$\bigsqcap \Psi = {\bf 1} \wedge x_1 \sqcup \dots \sqcup x_k \sqcup \overline{y}_1 \sqcup \dots \sqcup \overline{y}_l \neq {\bf 1}$ is unsatisfiable over
the 2-element Boolean algebra, which is the case
if and only if $\bigsqcap \Psi = {\bf 1} \wedge x_1 \sqcup \dots \sqcup x_k \sqcup \overline{y}_1 \sqcup \dots \sqcup \overline{y}_l = {\bf 0}$ is unsatisfiable over the two-element Boolean algebra. 
As we have seen in Lemma~\ref{lem:pre-inner-res}, 
this is turn holds if and only if Inner-Res$(\Psi {\cup} \{
\overline{x}_1,\dots,\overline{x}_k,y_1,\dots,y_l\})$
rejects.
\end{proof}

\begin{figure}[ht]
\begin{center}
\small
\hspace{-.4cm} 
\fbox{
\begin{tabular}{l}
Outer-Res($\Phi$) \\
{\rm // Input: A finite set $\Phi$ of Horn-Horn clauses} \\
{\rm // Accepts iff $\Phi$ is satisfiable over $({\cal P}({\mathbb N}); \sqcap,
\sqcup,c,{\bf 0},{\bf 1})$} \\
During the entire algorithm: \\
\hspace{.3cm} if $\Phi$ contains an empty clause, then reject. \\
Repeat := true \\
While Repeat = true do \\
\hspace{.3cm}  Repeat := false \\
\hspace{.3cm} Let $\Psi$ be the set of all inner Horn clauses of terms $t$ \\ 
\hspace{.6cm} from positive unit clauses $\{t={\bf 1}\}$ in $\Phi$. \\
\hspace{.6cm} If Inner-Res rejects $\Psi$, then reject. \\
\hspace{.3cm} For each negative literal $t \neq {\bf 1}$ in clauses  
from $\Phi$ \\
\hspace{.6cm} For each inner clause $D = \{x_1,\dots,x_k,\overline{y}_1,\dots,\overline{y}_l\}$  
of $t$ \\
\hspace{.9cm} Call Inner-Res on \\
\hspace{1.2cm} $\Psi \cup \{\overline{x}_1={\bf 1}, \dots, \overline{x}_k={\bf 1}, y_0={\bf 1}, \dots, y_l={\bf 1}\}$ \\
\hspace{.9cm} If Inner-Res rejects then remove clause $D$ from $t$\\
\hspace{.6cm} End for \\
\hspace{.6cm} If all clauses in $t$ have been removed, then \\
\hspace{.9cm} Remove outer literal $t \neq {\bf 1}$ from its clause \\
\hspace{.9cm} Repeat := true \\
\hspace{.3cm} End for \\
Loop \\
Accept
\end{tabular}}
\end{center}
\caption{Outer Resolution Algorithm.}
\label{fig:outer}
\end{figure}

\begin{thm}\label{thm:outer-res}
The algorithm `Outer-Res' in Figure~\ref{fig:outer} decides satisfiability for sets of Horn-Horn clauses in quadratic time.
\end{thm}
\begin{proof}
We first argue that if the algorithm rejects $\Phi$, then
$\Phi$ has indeed no solution. 
First note that during the whole argument, the set of clauses 
$\Phi$ has the same satisfying tuples (i.e. the corresponding formulas are equivalent): 
Observe that only negative literals get removed from clauses, and
that a negative literal $t \neq {\bf 1}$ only gets removed from a clause when Inner-Res rejects $\Psi \cup \{
\overline{x}_1={\bf 1},\dots,\overline{x}_k={\bf 1},y_0={\bf 1},\dots,y_l={\bf 1}\}$ for each 
inner clause $\{x_1,\dots,x_k,\overline{y}_1,\dots,\overline{y}_l\}$ of $t$. By Lemma~\ref{lem:inner-res}, if Inner-Res rejects 
$\Psi \cup \{\overline{x}_1={\bf 1},\dots,\overline{x}_k={\bf 1},y_0={\bf 1},\dots,y_l={\bf 1}\}$ then $\Psi$ implies that
$x_1 \sqcup \dots \sqcup x_k \sqcup \overline{y}_1 \sqcup \dots \sqcup \overline{y}_l = {\bf 1}$. 
Hence, the positive unit clauses imply that
$t = {\bf 1}$ and therefore the literal $t \neq {\bf 1}$ can 
be removed from the clause without 
changing the set of satisfying tuples. Now the algorithm rejects if either Inner-Res rejects $\Psi$ or if it derives the empty clause. In both 
cases it is clear that $\Phi$ is not satisfiable.

Thus, it suffices to construct a solution when the algorithm accepts. 
Let $\Psi$ be the set of all inner clauses of terms from positive
unit clauses at the final stage, when the algorithm accepts. 
For each remaining negative outer literal $\{t \neq {\bf 1}\}$
and each remaining inner clause $D = \{x_1,\dots,x_k,\overline{y}_1,\dots, \overline{y}_l\}$ of $t$ there exists an assignment $\alpha_D$
from $V \rightarrow {\cal P}(\mathbb{N})$ that satisfies $\Psi \cup \{x_1 \sqcup \dots \sqcup x_k \sqcup \overline{y}_1 \sqcup \dots \sqcup \overline{y}_l \neq {\bf 1}\}$: otherwise, by 
Lemma~\ref{lem:inner-res}, 
the inner resolution algorithm would have rejected
$\Psi \cup \{\overline{x}_1={\bf 1},\dots,\overline{x}_k={\bf 1},y_0={\bf 1},\dots,y_l={\bf 1}\}$, and would have removed the inner clause $D$ from $t$.
Let $D_1,\dots,D_s$ be an enumeration of all remaining inner clauses $D$ that appear in all remaining negative outer literals. 

Write $i_s$ for the $s$-ary operation defined by $
(x_1,\dots,x_s) \mapsto i(x_1,i(x_2,\dots,i(x_{s-1},x_s) \cdots ))$  (where $i$ is as in Fact~\ref{fact:nat-iso}).
We claim that $s \colon V \rightarrow {\cal P}(\mathbb{N})$ given by
\begin{align*} 
x \mapsto i_s(\alpha_{D_1}(x),\dots,\alpha_{D_s}(x))
\end{align*}
satisfies all clauses in $\Phi$.
Let $C$ 
be a clause from $\Phi$. By assumption, 
at the final stage of the algorithm, the clause $C$ is still non-empty.
Also note that since all formulas in the input were Horn-Horn,
they contain at most one positive literal.
This holds in particular for $C$, 
and we therefore only have to distinguish the following cases:
\begin{itemize}
\item At the final state of the algorithm, 
$C$ still contains a negative literal
$t \neq {\bf 1}$. Since $t \neq {\bf 1}$ has not been removed, there must be a remaining inner clause
$D = \{x_1,\dots, x_k,\overline{y}_1,\dots, \overline{y}_l\}$ of $t$.
Observe that $s(x_1) \sqcup \dots \sqcup s(x_k) \sqcup \overline{s(y_1)} \sqcup \dots
\sqcup \overline{s(y_l)} = {\bf 1}$ if and only if $\alpha_{D_j}(x_1) \sqcup \dots \sqcup \alpha_{D_j}(x_k) \sqcup \overline{\alpha_{D_j}(y_1)}
\sqcup \dots \sqcup \overline{\alpha_{D_j}(y_l)} = {\bf 1}$ for all $1 \leq j \leq s$. 
Hence, and since
$\alpha_{D}(x_1) \sqcup \dots \sqcup \alpha_{D}(x_k) \sqcup \overline{\alpha_{D}(y_1)} \sqcup \dots
\sqcup \overline{\alpha_{D}(y_l)} \neq {\bf 1}$,
$s$ satisfies $t \neq {\bf 1}$. This shows that $s$ satisfies $C$.
\item All negative literals have been removed from $C$ during the algorithm. The positive literal $t_0 = {\bf 1}$ of 
$C$ is such that the inner clauses of $t_0$ are Horn.
They will be part of $\Psi$, 
and therefore $t_0 = {\bf 1}$ is satisfied by
$s$. Indeed, by assumption the assignments $\alpha_{D_j}$ satisfy $\Psi$, and $\Psi$ is preserved by $i$.
\end{itemize}
We conclude that $s$ is a solution to $\Phi$.
The inner resolution algorithm has a linear time
complexity; the outer resolution algorithm 
performs at most a linear number of calls to the
inner resolution algorithm, and it is straightforward
to implement all necessary data structures for
outer resolution to obtain a running
time that is quadratic in the input size.
\end{proof}

Combining Proposition~\ref{prop:core-reduction}
with Theorem~\ref{thm:outer-res}, we obtain the following.

\begin{thm}\label{thm:tractability}
Let $\Gamma$ be a finite set constraint language from $\mathcal EI$. Then $\Csp(\Gamma)$ can be solved in quadratic time.
\end{thm}

\section{Maximality}
\label{sect:maximal}
In this section we show that the class $\mathcal EI$ is a maximal tractable set constraint language.
More specifically, let $\Gamma$ be a set constraint language that
strictly contains all $\mathcal EI$ relations. 
We then show that
$\Gamma$ contains a finite set of relations $\Gamma'$ such that
already the problem $\Csp(\Gamma')$ is NP-hard (Theorem~\ref{thm:set-maximal}).

In our proof we use the so-called \emph{universal-algebraic approach} to the complexity of constraint satisfaction problems, 
which requires
that we re-formulate set CSPs as constraint satisfaction problems
for \emph{$\omega$-categorical} structures. For a more detailed
introduction to the universal-algebraic approach for $\omega$-categorical structures, see~\cite{Bodirsky-HDR}. 
A structure $\Gamma$ with a countable domain is called 
\emph{$\omega$-categorical}
if all countable structures that satisfy the same first-order sentences
as $\Gamma$ are isomorphic to $\Gamma$ (see e.g.~\cite{HodgesLong}). 
By the theorem of Ryll-Nardzewski, and for countable signatures, this is equivalent to requiring that every relation that is preserved by the automorphisms\footnote{An isomorphism of a structure $\Gamma$ with itself
is called an \emph{automorphism} of $\Gamma$.} of
$\Gamma$ is first-order definable in $\Gamma$ (see e.g.~\cite{HodgesLong}). The set of all automorphisms of $\Gamma$ is denoted by $\Aut(\Gamma)$.

It is well-known that all countable atomless\footnote{An \emph{atom} in a Boolean algebra is an element $x \neq {\bf 0}$ such that for all
$y$ with $x \cap y = y$ and $x \neq y$ we have $y=\bf 0$.
If a Boolean algebra does not contains atoms, it is called \emph{atomless}.} 
Boolean algebras are isomorphic (Corollary 5.16 in~\cite{KoppelbergBoolenAlgebras}; also see Example 4 on 
page 100 in~\cite{HodgesLong}); 
let $\mathfrak A$ denote such a countable atomless Boolean algebra.
Let $\mathbb A$ denote the domain of $\mathfrak A$. 
Again, we use $\sqcap$ and $\sqcup$ to denote join and meet in $\mathfrak A$, respectively.
Since the axioms of Boolean algebras and the property of not having atoms can all be written as first-order sentences, it follows that $\mathfrak A$ is $\omega$-categorical. 
A structure $\mathfrak B$ has \emph{quantifier elimination} if every first-order formula is over $\mathfrak B$ equivalent to a quantifier-free formula. 
It is well-known that $\mathfrak A$ has quantifier elimination (see Exercise 17 on Page 391 in~\cite{HodgesLong}). 
We will also make use of the following.

\begin{thm}[Corollary 5.7 in~\cite{MarriottOdersky}]
\label{thm:mo}
A quantifier-free formula is satisfiable in some infinite Boolean algebra if and only if it is satisfiable in all infinite Boolean algebras.
\end{thm}

A fundamental concept in the complexity theory of constraint
satisfaction problems is the notion of \emph{primitive positive definitions}. A first-order formula is called \emph{primitive positive (pp)} 
if it is of the form $$ \exists x_1,\dots,x_n \, (\psi_1 \wedge \dots \wedge \psi_m)$$
where for each $i \leq m$ the formula $\psi_i$ is of the form
$R(y_1,\dots,y_l)$ or of the form $y_1 = y_2$, and where $R$ is a relation symbol and $y_1,y_2,\dots,y_l$ are either free variables or from $\{x_1,\dots,x_n\}$. We say that a $k$-ary relation $R \subseteq D^k$
is \emph{primitive positive definable (pp definable)} over a $\tau$-structure $\Gamma$ with domain $D$ iff there exists a primitive positive
formula $\phi(x_1,\dots,x_k)$ with the $k$ free variables $x_1,\dots,x_k$ such that a tuple $(b_1,\dots,b_k)$ is in $R$
if and only if $\phi(b_1,\dots,b_k)$ is true in $\Gamma$.

{\bf Example.} The relation 
$\{(x,y) \in {\cal P}({\mathbb N})^2 \; | \; x \sqsubset y \}$ is pp definable in $({\cal P}({\mathbb N}); S, \neq)$ where  $S = \{(x,y,z) \; | \; x \sqcap y \sqsubseteq z\}$. The pp definition is 
$S(x,x,y) \wedge x \neq y$ (the definition is even quantifier-free). \qed

{\bf Example.} The relation 
$\{(x_1,x_2,x_3,y) \in {\cal P}({\mathbb N})^4 \; | \; x_1 \sqcap x_2 \sqcap x_3 \sqsubseteq y \}$ is pp definable in $({\cal P}({\mathbb N}); S)$ where  $S = \{(x,y,z) \; | \; x \sqcap y \sqsubseteq z\}$. The pp definition is $\exists u \; (S(x_1,x_2,u) \wedge S(u,x_3,y))$. \qed

When
every relation of a structure $\Gamma$ is preserved by an operation $f$, then $f$ is called a \emph{polymorphism} of $\Gamma$. Note that polymorphisms of $\Gamma$ also preserve all relations that have a pp definition in $\Gamma$. 
The following has been shown for
finite domain constraint satisfaction in~\cite{JBK}; the easy
proof also works for infinite domain constraint satisfaction.

\begin{lem}\label{lem:pp-reduce}
Let $R$ be a relation with a primitive positive definition in 
a structure $\Gamma$. Then $\Csp({\Gamma})$
and the CSP for the expansion of $\Gamma$ by 
the relation $R$ are polynomial-time equivalent.
\end{lem}

The following theorem is one of the reasons why it
is useful to work with $\omega$-categorical templates 
(when this is possible). 

\begin{thm}[from~\cite{BodirskyNesetrilJLC}]\label{thm:pp-pres}
Let $\Gamma$ be an $\omega$-categorical structure.
Then $R$ is primitive positive definable in $\Gamma$ 
if and only if $R$ is preserved by all polymorphisms of $\Gamma$.
\end{thm}

The previous and the next result together can be used to
translate questions about primitive positive definability into 
purely operational questions.
Let $D$ be a set, 
let ${\cal O}^{(n)}$ be $D^n \rightarrow D$, 
and let $\mathcal O$ be $\bigcup_{n=1}^\infty {\cal O}^{(n)}$ the set of operations on $D$ of finite arity.
An operation $\pi \in {\cal O}^{(n)}$ is called a \emph{projection} 
if for
some fixed $i \in \{1, \dots, n\}$ and 
for all $n$-tuples $(x_1,\dots,x_n) \in D^n$ we have the identity 
$\pi(x_1, \dots, x_n) = x_i$.
 The {\em composition} of a $k$-ary operation $f$ and $k$ operations
$g_1,\dots, g_k$ of arity $n$ is the $n$-ary operation defined by
\begin{align*}
& (f(g_1,\dots,g_k))(x_1,\dots,x_n) \\ 
= \quad & f\big(g_1(x_1,\dots,x_n),\dots,g_k(x_1,\dots,x_n)\big).\ \ 
\end{align*}

\begin{defn}
We say that ${\mathcal F} \subseteq \mathcal O$ \emph{locally generates} $f \colon D^n \rightarrow D$
if for every finite subset $A$ of $D^n$ there is an operation $g \colon D^n \rightarrow D$
that can be obtained from the operations in $\mathcal F$
and projection maps by composition such that $f(a)=g(a)$ for all $a \in A$.
\end{defn}

\begin{thm}[see~\cite{Szendrei}]\label{thm:loc-clos}
Let $\mathcal F \subseteq {\mathcal O}$ be a set of operations with domain $D$.
Then an operation $f \colon D^k \rightarrow D$ preserves all finitary
relations that are preserved by all operations in $\mathcal F$
if and only if $\mathcal F$ locally generates $f$.
\end{thm}

In the following, we always consider sets of operations $\mathcal F$ that contain $\Aut(\mathfrak A)$, and therefore make the following convention. For ${\mathcal F} \subseteq {\mathcal O}$, we say
that $\mathcal F$ \emph{generates} $f \in {\mathcal O}$ if $\mathcal F \cup \text{Aut}(\mathfrak A)$ locally generates $f$.
We now define analogs of the operations $e$ and $i$, defined on $\mathbb A$ instead of ${\cal P}({\mathbb N})$.

\begin{prop}\label{prop:iso}
There is an isomorphism $\tilde i$ between ${\mathfrak A}^2$ and $\mathfrak A$.
\end{prop}
\begin{proof}
It is straightforward to verify that ${\mathfrak A}^2$ is again
a countable atomless Boolean algebra.
\end{proof}

Motivated by the properties of $e$ described in Lemma~\ref{lem:e}, we make the following definition.

\begin{defn}
Let $\mathfrak B$ and ${\mathfrak B}'$ be two arbitrary
Boolean algebras with domains $B$ and $B'$,
respectively, and let 
$g \colon B \to B'$ be a function 
 that strongly preserves $\sqcap$, ${\bf 0}$, and ${\bf 1}$. 
  We say that
 $g$ \emph{forgets unions} if
 for all $k \geq 1$, $l \geq 0$, and
$x_1,\dots,x_k,y_1,\dots,y_l \in B$ 
 we have $$e(x_1) \sqcup \cdots \sqcup e(x_k) \sqcup \overline{e(y_1)} \sqcup \cdots \sqcup \overline{e(y_l)} = {\bf 1}$$ if and only if 
 there exists an $i \leq k$ such that 
$x_i \sqcup \overline{y_1} \sqcup \cdots \sqcup \overline{y_l} = {\bf 1}$.
\end{defn}


\begin{prop}~\label{prop:e-atomless-exists}
There exists an injection $\tilde e \colon {\mathbb A} \rightarrow {\mathbb A}$ that 
strongly preserves $\sqcap$, ${\bf 0}$, and ${\bf 1}$ 
in $\mathfrak A$, and that forgets unions.
\end{prop}
\begin{proof}
The construction of $\tilde e$ is by a standard application of K\"onig's tree lemma for $\omega$-categorical structures (see e.g.~\cite{Bodirsky-HDR}); 
it suffices to show that there is an injection $f$ 
from every finite induced substructure
$\mathfrak B$ of $\mathfrak A$ to $\mathfrak A$ such that $f$ strongly preserves $\sqcap$, ${\bf 0}$, and ${\bf 1}$, and forgets unions. 

So let $\mathfrak B$ be such a finite substructure of $\mathfrak A$, and let $B$ be the domain of $\mathfrak B$.
Let ${\mathfrak C} = ({\cal P}(B); \sqcap, \sqcup, c,{\bf 0},{\bf 1})$ be the Boolean
algebra of the subsets of $B$. 
We claim that $g \colon B \rightarrow {\cal P}(B)$ given by
$g({\bf 1}) = {\bf 1}$ and $g(x) = \{ z \; | \; z \neq {\bf 0} \wedge z \sqsubseteq^{\mathfrak B} x\}$ for $x \neq {\bf 1}$
\begin{itemize}
\item preserves ${\bf 0}$ and ${\bf 1}$: this is by definition;
\item preserves $\sqcap$: for $x,y \in B$ (including the case that $x = {\bf 1}$ or $y = {\bf 1}$) we have 
\begin{align*}
g(x) \sqcap^{\mathfrak C} g(y) \, = & \, \{z \; | \; z \neq {\bf 0} \wedge z \sqsubseteq^{\mathfrak B} x \wedge z \sqsubseteq^{\mathfrak B} y\} \\
 = & \, \big \{ z \; | \; z \neq {\bf 0} \wedge z \sqsubseteq^{\mathfrak B} (x \sqcap^{\mathfrak B} y) \big \} \\
 = & \, g(x \sqcap^{\mathfrak B} y) \; ;
\end{align*}
\item is injective: if $x,y \in B$ such that $g(x)=g(y)$, then $x \sqsubseteq^{\mathfrak B} y$ and $y \sqsubseteq^{\mathfrak B} x$, and hence $x=y$; 
\item strongly preserves $\sqcap$: this follows from the previous two items;
\item forgets unions: 
This can be shown analogously to the proof of
Lemma~\ref{lem:e}. 
\end{itemize}
%
Clearly, there is an embedding $h$ from $\mathfrak C$ into $\mathfrak A$. 
Then $f:=h(g)$
is a homomorphism from $\mathfrak B$ to $\mathfrak A$ that forgets unions.
\end{proof}

\begin{prop}~\label{prop:e-atomless-correct}
Let $\phi$ be a quantifier-free first-order formula over the signature
$\{\sqcap,\sqcup,c,{\bf 0},{\bf 1}\}$. Then $e$ preserves $\phi$
over $\mathfrak S$ if and only if $\tilde e$ preserves $\phi$
over $\mathfrak A$. Moreover, every operation from ${\mathbb A} \rightarrow {\mathbb A}$ that strongly preserves
$\sqcap$, $\bf 0$, and $\bf 1$ and forgets unions generates
$\tilde e$, and is generated by $\tilde e$.
\end{prop}
\begin{proof}
Let $\bar a$ be a tuple of elements from $\mathbb A$.
Clearly, there exists a tuple $\bar b$ of elements from ${\cal P}({\mathbb N})$ such that $\bar a$ and $\bar b$ satisfy the same set $\psi$ of quantifier-free first-order formulas; this follows from the fact that every finite Boolean algebra is the Boolean algebra of subsets of a finite set.
Now observe that whether or not 
the tuple $e(\bar b)$
satisfies a quantifier-free first-order formula $\phi$ 
only depends on $\psi$, by Lemma~\ref{lem:e}. 
Since $\tilde e$ strongly preserves $\sqcap$, ${\bf 0}$, and ${\bf 1}$, and forgets unions, 
the same is true
for the quantifier-free first-order formulas that hold on $\tilde e(\bar a)$.
Hence, $\tilde e$ preserves $\phi$ over $\mathfrak A$ if and only if
$e$ preserves $\phi$ over $\mathfrak S$. 

To prove the second part of the statement, we use Theorem~\ref{thm:loc-clos}.
 Suppose that $\bar c$ and $\bar d$ are tuples
of elements from ${\mathbb A}$ that 
satisfy the same quantifier-free first-order formulas. By the equivalent characterization 
of $\omega$-categoricity mentioned above,
and the fact that ${\mathfrak A}$ has quantifier-elimination, there exists an automorphism
$\alpha$ of ${\mathfrak A}$ that maps 
$\bar c$ to $\bar d$. By the above observations
and Theorem~\ref{thm:loc-clos}, this implies that all operations that strongly preserve
$\sqcap$, $\bf 0$, and $\bf 1$, and forget unions generate each other.
\end{proof}

Let $\widetilde{ei}$ be the operation $(x,y) \mapsto \tilde e(\tilde i(x,y))$.  The following can be shown similarly to Proposition~\ref{prop:e-atomless-correct}. 

\begin{prop}~\label{prop:ei-atomless-correct}
Let $\phi$ be a quantifier-free first-order formula over the signature
$\{\sqcap,\sqcup,c,{\bf 0},{\bf 1}\}$. Then $ei$ preserves $\phi$
over $\mathfrak S$ if and only if $\widetilde{ei}$ preserves $\phi$
over $\mathfrak A$. Moreover, every binary operation $g$ that strongly
preserves $\sqcap$, ${\bf 0}$, and ${\bf 1}$, and forgets unions generates $\widetilde{ei}$, and is generated by $\widetilde{ei}$. \qed
\end{prop}

We now give the central argument for the maximal tractability of $\mathcal EI$, stated in
universal-algebraic language.
We say that an operation from ${\mathbb A}^k \rightarrow {\mathbb A}$ \emph{depends} on the argument
$i \in \{1,\dots,k\}$ if there is no $(k{-}1)$-ary operation $f'$ such that for all $x_1,\dots,x_k \in {\mathbb A}$
$$f(x_1,
\dots, x_k) = f'(x_1, \dots, x_{i-1}, x_{i-1}, x_{i+1},\dots, x_k)\;.$$ 
We
can equivalently characterize $k$-ary operations that depend on the
$i$-th argument by requiring that there are $x_1, \dots,
x_k \in {\mathbb A}$ and $x_i' \in {\mathbb A}$ such that $$f(x_1, \dots, x_k) \neq f(x_1, \dots,
x_{i-1}, x'_i, x_{i+1}, \dots, x_k) \; .$$

\begin{thm}\label{thm:minimal}
Let $f$ be an operation generated by $\{\widetilde{ei}\}$. 
Then either $\{f\}$ generates $\widetilde{ei}$, or 
$f$ is generated by $\{\tilde e \}$.
\end{thm}
\begin{proof}
To show the statement of the theorem, 
let $f$ be a $k$-ary operation generated by $\{\widetilde{ei}\}$.
For the sake of notation, let $x_1,\dots,x_l$ be the arguments on which $f$ depends, for $l \leq k$. Let
$f' \colon {\mathbb A}^l \rightarrow {\mathbb A}$ be the operation given by $f'(x_1,\dots,x_l) = f(x_1,\dots,x_l,\dots,x_l)$. Clearly, $f'$ must be injective
(since it is generated from an injective operation and depends on all arguments).
Since $f'$ is generated
by $\widetilde{ei}$ it preserves $\sqcap$, ${\bf 0}$, ${\bf 1}$, and since $f'$ is injective,
it also \emph{strongly} preserves those functions. 

Consider first the case that $l=1$, i.e., $f'$ is unary. 
If for all finite subsets of $\mathbb A$,
the operation $f'$ equals an automorphism of $\mathfrak A$, then
$f$ is generated by $\Aut(\mathfrak A)$ and there is nothing to show.
So assume otherwise; that is,
assume that there is a finite set $S \subseteq {\mathbb A}$ such that there is no
$a \in \Aut(\mathfrak A)$ with $f'(x) = a(x)$ for all $x \in S$. We claim that 
$f'$ forgets unions. 
To see this, let $u_1,\dots,u_m,v_1,\dots,v_n$ be from $\mathbb A$ such that
$f'(u_1) \sqcup \cdots \sqcup f'(u_m) \sqcup \overline{f'(v_1)} \sqcup \cdots \sqcup \overline{f'(v_n)} = {\bf 1}$.
Since $f'$ is generated by $\{\tilde{ei}\}$,
there is a term composed from automorphisms of $\mathfrak A$ and $\widetilde{ei}$ such that $f'(x) = T(x)$ for all $x \in S \cup \{u_1,\dots,u_m,v_1,\dots,v_n\}$. By the choice of $S$, this term cannot be composed of automorphisms alone, and hence
there must be $a \in \Aut(\mathfrak A)$ and operational terms $T_1,T_2$ composed from automorphisms of $\mathfrak A$ and $\widetilde{ei}$ such that $f'(x) = a(\widetilde{ei}(T_1(x),T_2(x)))$ for all $x \in S$. 
As $\widetilde{ei}$ forgets unions, there exists an $i \leq k$
such that $T_1(u_i) \sqcup \overline{T_1(v_1)} \sqcup \cdots \sqcup \overline{T_l(v_n)} = {\bf 1}$. Since $T_1$ strongly preserves $\sqcap$ we conclude
that there exists an $i$ such that
$u_i \sqcup \overline{v}_1 \sqcup \cdots \sqcup \overline{v}_n = {\bf 1}$  (see the proof of Proposition~\ref{prop:inner-horn}), which is what we wanted to show. 
By Proposition~\ref{prop:e-atomless-correct} it follows that $f'$ is generated by
$\tilde e$. But then $f$ is generated by $\tilde e$ as well. 

Next, consider the case that $l>1$. Let $g$ be the binary operation
defined by $g(x,y) = f'(x,y,\dots,y)$; since $f'$ is injective, 
the operation $g$ will also be injective, and in particular depends on both arguments,
and strongly preserves ${\bf 0}$, ${\bf 1}$, and $\sqcap$. We claim that
$g$ forgets unions. 
Let $u_1 = (u^1_1,u_1^2),\dots, u_m =(u^1_m,u^2_m), v_1 = (v_1^1,v_1^2),\dots,v_n = (v_n^1,v^2_n)$ be from ${\mathbb A}^2$ such that
$g(u_1) \sqcup \cdots \sqcup g(u_m) \sqcup \overline{g(v_1)} \sqcup \cdots \sqcup \overline{g(v_n)} = {\bf 1}$.
Since $g$ is generated by $\widetilde{ei}$
and cannot be generated by the automorphisms of $\mathfrak A$ alone,
there is a term of the form 
$T(x,y) = a(\widetilde{ei}(T_1(x,y),T_2(x,y)))$ where 
\begin{itemize}
\item $a \in \text{Aut}(\mathfrak A)$, 
\item $T_1$ and $T_2$ are operational terms composed from automorphisms 
of $\mathfrak A$ and $\widetilde{ei}$,
\item $g(x,y) = T(x,y)$ for all 
$(x,y) \in \{u_1,\dots,u_m,v_1,\dots,v_n\}$. 
\end{itemize}
Since $\widetilde{ei}$ forgets unions,
there exists an $i \leq k$ such that
$T_1(u_i) \sqcup \overline{T_1(v_1)} \sqcup \cdots \sqcup \overline{T_1(v_n)} = {\bf 1}$
and $T_2(u_i) \sqcup \overline{T_2(v_1)} \sqcup \cdots \sqcup \overline{T_2(v_n)} = {\bf 1}$. Suppose first that $T_1$ depends on both arguments. Then $T_1$ defines an injective
operation and strongly preserves $\sqcap$. It follows 
that $u_i \sqcup \overline{v}_1 \sqcup \cdots \sqcup \overline{v}_n = {\bf 1}$ in ${\mathfrak A}^2$ since these equations are inner Horn.
We can argue similarly if $T_2$ depends on
both arguments, and in those cases 
we have established that $g$ forgets unions. 
Suppose now that each of $T_1$ and $T_2$
does \emph{not} depend on both arguments. Consider 
first the case that $T_1$ only depends on the first argument. Then the function $x \mapsto T_1(x,x)$ is injective and strongly preserves $\sqcap$, and 
from $T_1(u_i) \sqcup \overline{T_1(v_1)} \sqcup \cdots \sqcup \overline{T_1(v_n)} = {\bf 1}$
we derive as above that $u^1_i \sqcup \overline{v^1_1} \sqcup \cdots \sqcup \overline{v^1_n} = {\bf 1}$ holds in ${\mathfrak A}$. In this case, $T_2$ must depend on the second argument, since $T$ depends on both arguments. We therefore also have
that $u^2_i \sqcup \overline{v^2_1} \sqcup \cdots \sqcup \overline{v^2_n} = {\bf 1}$ holds in ${\mathfrak A}$. 
The situation that 
$T_1$ only depends on the second argument
and $T_2$ only depends on the first argument is analogous. So $g$ forgets unions. 
By Proposition~\ref{prop:ei-atomless-correct}, $g$ generates 
$\widetilde{ei}$. Consequently, also $f$
generates $\widetilde{ei}$.
\end{proof}

\begin{defn}
Let $U, I \subseteq {\mathbb A}^3$ be the following relations
\begin{align*}
U := & \{(x,y,z)  \; | \; (x \sqcup y = z) \} \\
I := & \{ (x,y,z)  \; | \; (x \sqcap y = z) \}
\end{align*}
\end{defn}

Note that both $U$ and $I$ are preserved by $\tilde i$. 
The following demonstrates that when $\Gamma$ has the polymorphism
$\tilde i$, this does not suffice for tractability of $\Csp(\Gamma)$.

\begin{prop}\label{prop:hard}
Let $\Gamma$ be the structure with domain $\mathbb A$ and three
relations $U$, $I$, and $\neq$. Then $\Csp(\Gamma)$ is NP-hard.
\end{prop}
\begin{proof}
The proof is by reduction from 3SAT. We compute from
a given 3SAT instance $\Phi$ with variable set $V$ an instance $\Psi$ of $\Csp({\Gamma})$ (in polynomial time) as follows. There are distinguished variables $t$ and $f$ in $\Psi$. 
For each variable $x$ of $\Phi$ there are two variables 
$x_t$ and $x_f$ in $\Psi$.
For a clause $C := \{l_1,l_2,l_3\}$ of $\Phi$ (with literals $l_i$ either of the form $x$ or of the form $\neg x$) we create a fresh variable $u_C$, and add the constraints
$U(v_1,v_2,u_C)$ and $U(u_C,v_3,t)$ where $v_i := x_f$ if $l_i = \neg x$, and $v_i := x_t$ if $l_i = x$. Moreover, we add for each variable
$x \in V$ the constraints $U(x_t,x_f,t)$ and $I(x_t,x_f,f)$. 
Finally, add the constraint $t \neq f$ and $I(t,f,f)$. 

It is clear that if $\Phi$ has the satisfying assignment $\alpha \colon V \rightarrow \{0,1\}$ then the following assignment $\beta$ satisfies
all constraints in $\Psi$. 
Choose $S_f \subsetneq S_t$ arbitrarily. 
Then $\beta$ maps $t$ to $S_t$ and $f$ to $S_f$, 
it maps $x_t \in V$ to $S_t$ if $\alpha(x) = 1$
and to $S_f$ otherwise, 
it maps $x_f \in V$ to $S_f$ if $\alpha(x) = 1$
and to $S_t$ otherwise, 
and for every clause 
$C=\{l_1,l_2,l_3\}$ of $\Phi$ it maps 
$u_C$ to $S_t$ if $\alpha(l_1)=1$ or $\alpha(l_2)=1$, and to $S_f$ otherwise.

Conversely, suppose that $\beta$ maps the variables of $\Psi$ 
to the elements of $\Gamma$ satisfying all constraints of $\Psi$. 
Let $\mathfrak B$ be
the finite Boolean algebra that is generated 
by $\beta(V)$ in $\Gamma$. Since $\beta(f) \subsetneq \beta(t)$, 
we have that $\beta(t)$ is non-empty.
Select an arbitrary atom $a$ of $\mathfrak B$ 
that is contained in $\beta(t)$. 
Then we set $\alpha(x)$ for $x \in V$
to $1$ if $a \subseteq \beta(x_t)$ and to $0$ otherwise. In this way
all clauses $\{l_1,l_2,l_3\}$ of $\Phi$ are satisfied.
To see this, assume for simplicity of presentation that $l_1 = \neg x$ is negative and $l_2 = y$ and $l_3 = z$ are positive; the general case is analogous. 
Since we have the constraints $U(x_f,y_t,u_C)$ and $U(u_C,z_t,t)$,
and since $a$ is an atom of $\mathfrak B$, 
one of $\beta(x_f),\beta(y_t),\beta(z_t)$ must contain $a$.
If $a$ is in $\beta(y_t)$ or $\beta(z_t)$ then $\alpha(y)$ or $\alpha(z)$ is set to $1$. If $a$ is in $\beta(x_f)$, then the clause $I(x_t,x_f,f)$
forces that $a$ is not in $\beta(x_t)$, and hence $\alpha(x)$ is set to $0$. Thus, $\alpha$ sets at least one of $\neg x,y,z$ to $1$, and the clause $C$ is satisfied.
\end{proof}

\begin{thm}\label{thm:set-maximal}
Let $\Gamma$ be a set constraint language.
Suppose that $\Gamma$ contains all relations from $\mathcal EI$, and also contains a relation that is not from $\mathcal EI$. Then there is a finite sublanguage $\Gamma'$ of $\Gamma$ such that  $\Csp(\Gamma')$ is NP-hard.
\end{thm}
\begin{proof}
When $R_1,R_2,\dots$ are the relations of $\Gamma$,
let $\phi_1,\phi_2,\dots$ be quantifier-free first-order formulas
that define $R^{\Gamma}_1,R^{\Gamma}_2,\dots$ over $\mathfrak S = ({\cal P}({\mathbb N}); \sqcup,\sqcap,c,{{\bf 0}},{{\bf 1}})$. Let $R^{\mathfrak A}_1,R^{\mathfrak A}_2,\dots$ be the
relations defined by $\phi_1,\phi_2,\dots$ over $\mathfrak A$,
and let $\Delta$ be the relational structure with domain $\mathbb A$ and exactly those relations.
By Proposition~\ref{prop:ei-atomless-correct}, $\Delta$ 
contains a relation that is not preserved by $\widetilde{ei}$,
and contains all relations that are preserved by $\widetilde{ei}$.
Consider the set $\mathcal F$ of all polymorphisms of $\Delta$.

The set $\mathcal F$ does not contain $\widetilde ei$, since this would contradict by Theorem~\ref{thm:loc-clos} the fact that $\Delta$ contains a relation that is not preserved by $\widetilde{ei}$. Since $\mathcal F$ is locally closed, it 
follows from Theorem~\ref{thm:minimal} that
all operations $f \in \mathcal F$ are generated by $\tilde e$.
But then the relation $\{(x,y,z) \; |  \; x=y \neq z \vee x \neq y = zÊ\}$
is preserved by all operations in $\mathcal F$,
and hence pp definable in $\Gamma$ by Theorem~\ref{thm:pp-pres}.
This relation has an NP-complete CSP~\cite{ecsps}. 
Let $\Delta'$ be the reduct of $\Delta$ that contains exactly the relations that appear in those pp definitions. Clearly, there are finitely many such relations; we denote the corresponding relation symbols by $\tau' \subset \tau$. 
By Lemma~\ref{lem:pp-reduce}, $\Csp(\Delta')$ is NP-hard. 

This establishes also the hardness of $\Csp(\Gamma)$:
let $\Gamma'$ be the $\tau'$-reduct of $\Gamma$.
We claim that $\Csp(\Gamma')$ and $\Csp(\Delta')$ are the same computational problem. We have to show that a conjunction 
of atomic $\tau'$-formulas $\Phi$ is satisfiable in $\Gamma'$ 
if and only if it is true in $\Delta'$. Replacing each atomic $\tau'$-formula in $\Phi$ by its quantifier-free first-order definition, 
this follows from Theorem~\ref{thm:mo}. 
\end{proof}

\section{Concluding Remarks}
We have introduced the powerful set constraint language of 
$\mathcal EI$ set constraints, which in particular contains
all Horn-Horn set constraints and all previously studied tractable set constraint languages. Constraint satisfaction problems 
over $\mathcal EI$ 
can be solved in polynomial -- even quadratic -- time. Our
tractability result is complemented by a complexity result
which shows that tractability of $\mathcal EI$ set constraints
is best-possible within a large class of set constraint languages. 

It is not hard to see from the properties we prove 
for $\mathcal EI$ set constraints that there is an algorithm to test whether a given finite set constraint language
(where relations in the language are given by quantifier-free formulas over the signature $\{\sqcup,\sqcap,{\bf c},{\bf 0}, {\bf 1} \}$) is contained in $\mathcal EI$. This means that
the so-called \emph{meta-problem} for $\mathcal EI$ 
set constraints can be decided effectively. 

We would also like to remark that one can analogously obtain tractability for the class of 
constraints where the inner clauses of the positive outer literals are \emph{dual Horn} (i.e., have at most one \emph{negative} literal). 

\paragraph{Acknowledgements}
We want to thank Fran\c{c}ois Bossi\`ere who pointed out mistakes
 in the conference
version of the paper. 




\bibliography{global.bib}
\bibliographystyle{alpha} 

\end{document}